\documentclass[letterpaper]{article} 
\usepackage{aaai2026}
\nocopyright
\usepackage{times}  
\usepackage{helvet}  
\usepackage{courier}  
\usepackage[hyphens]{url}  
\usepackage{graphicx} 
\usepackage{natbib}  
\usepackage{caption} 
\usepackage{algorithm}
\usepackage{algorithmic}

\usepackage{newfloat}
\usepackage{listings}
\DeclareCaptionStyle{ruled}{labelfont=normalfont,labelsep=colon,strut=off} 
\floatstyle{ruled}
\newfloat{listing}{tb}{lst}{}
\floatname{listing}{Listing}
\title{AAAI Press Anonymous Submission\\Instructions for Authors Using \LaTeX{}}
\author {
    Junli Jiang\textsuperscript{\rm 1},
    Pavel Naumov\textsuperscript{\rm 2}
}
\affiliations {
    \textsuperscript{\rm 1} Institute of Logic and Intelligence, Southwest University, China\\
    \textsuperscript{\rm 2} University of Southampton, United Kingdom\\
    walk08@swu.edu.cn,  
    p.naumov@soton.ac.uk
}

\title{Higher-Order Responsibility}

\usepackage{amssymb,amsmath}
\newtheorem{definition}{Definition}
\newtheorem{theorem}{Theorem}
\newtheorem{lemma}{Lemma}
\newtheorem{claim}{Claim}

\newenvironment{proof}{\begin{trivlist}\item\noindent{\sc Proof.}}{\hfill$\Box\hspace{2mm}$\end{trivlist}}

\newenvironment{proof-of-claim}{\begin{trivlist}\item\noindent{\sc Proof of Claim.}}{\hfill $\boxtimes\hspace{2mm}$\end{trivlist}}

\renewcommand{\phi}{\varphi}
\renewcommand{\epsilon}{\varepsilon}
\usepackage{booktabs}
\usepackage{multirow}
\usepackage{makecell}
\usepackage{graphicx}
\newcommand{\GF}{{\sf GF}}

\newcommand{\R}{{\sf R}}
\newcommand{\G}{{\sf G}}

\begin{document}

\date{}

\maketitle

\begin{abstract}
In ethics, individual responsibility is often defined through Frankfurt's principle of alternative possibilities. This definition is not adequate in a group decision-making setting because it often results in the lack of a responsible party or ``responsibility gap''. One of the existing approaches to address this problem is to consider group responsibility. Another, recently proposed, approach is  ``higher-order'' responsibility. The paper considers the problem of deciding if higher-order responsibility up to degree $d$ is enough to close the responsibility gap. The main technical result is that this problem is $\Pi_{2d+1}$-complete.
\end{abstract}


\section{Introduction}

Autonomous agents--from self-driving cars and stock traders to military robots and medical assistants--are increasingly involved in decisions that impact human lives. These decisions often result from collaborations among multiple parties in hybrid human-machine environments. Assigning individual responsibility for the outcomes of such decisions to at least one party fosters accountability and enhances trust.

\subsubsection*{Counterfactual Responsibility}

As a concept, responsibility has been studied in philosophy, law, and, more recently, artificial intelligence. In philosophy, the starting point of the discussion of responsibility is often the so-called {principle of alternative possibilities}: ``{\em 
 a person is morally responsible for what he has done only if he could have done otherwise}''~\cite{f69tjop}. The arguments for and against this principle have been proposed~\cite{b71jp,n84ps,c97mous,w17,mr18ar}. Even Frankfurt argued that its applicability is limited. In this paper, we use the term {\em counterfactual responsibility} to refer to the definition through the principle of alternative possibilities. Following recent literature in artificial intelligence~\cite* {ydjal19aamas,nt19aaai,nt20aaai,bfm21ijcai,s24aaai}, we interpret ``could have done'' as having a {strategy} that would {\em guarantee} the prevention of the undesired outcome no matter what the other agents do. 
 
As an example, imagine two factories, A and B, that can potentially dump a pollutant into a river inhabited by fish~\cite{h16}. Assume that it is well-known that it takes 15kg of pollutants to kill the fish. 
Consider the case when factories A and B have accumulated 20kg and 10kg of the pollutant, respectively. Suppose they both dump the pollutant into the river. The fish is dead. Which of the two factories is responsible for the death of the fish? Note that factory A has had a strategy (``do not dump the pollutant'') that guarantees that the fish would have stayed alive no matter what the other factory does. Factory B has not had such a strategy. Thus, only factory A is counterfactually responsible for the death of the fish. 

\subsubsection*{Deontic Constraint}

Let us now consider a different case in which factories A and B have each accumulated only 10kg of the pollutant. Thus, the fish will die only if both of the factories dump the pollutant.
In such a situation, it would be acceptable for either of the factories, but not both, to dump the pollutant. We use propositional variables $p_a$ (``pollute $a$'') and $p_b$ (``pollute $b$'') to represent the actions of factories A and B, respectively. If $p_x=1$, then factory X dumps the pollutant. If $p_x=0$, it does not dump. We use the {\em deontic constraint} formula $\gamma=\neg(p_a\wedge p_b)$ to express which combinations of actions are acceptable. In our case, the deontic constraint $\gamma$ states that it is not acceptable for both companies to dump the pollutant.  Let us further suppose that both factories have dumped the pollutant, and the fish is dead. Who is responsible? In this setting, each of the two companies has had a strategy (``do not dump the pollutant'') to save the fish. Thus, both of them are counterfactually responsible for the death of the fish.

\subsubsection*{Responsibility Gap}

Next, assume that each of the two factories, A and B, has accumulated 20kg of the pollutant. Thus, the fish dies if either of the factories dumps the pollutant. In this situation, the deontic constraint is $\gamma=\neg(p_a\vee p_b)$. In our previous examples, it is not important if the factories make decisions simultaneously (independently) or consequently. The situation is different in this example. If the decisions are made simultaneously, then neither of the two factories has a strategy that {\em guarantees} that the fish is alive. Thus, each time when either of the factories dumps the pollutant, the fish dies and neither of the factories is counterfactually responsible for the death. By an {\em action profile} we mean the pair $(p_a,p_b)$ that describes the actions of both factories. By the {\em responsibility gap} we mean the set of all profiles under which the deontic constraint $\gamma$ is violated and none of the agents is counterfactually responsible for this. In our example, the responsibility gap is 
$\{(1,0),(0,1),(1,1)\}$.  

The situation is different if the factories make their actions not simultaneously, but sequentially. Suppose that factory B makes its decision after factory A. Furthermore, we assume that factory B knows the choice made by factory A. Consider action profile $(0,1)$. Note that in this situation factory B knows that factory A decided not to dump the pollutant. Thus, it now has a strategy (``do not dump'') to save the fish. Nevertheless, under the profile $(0,1)$, factory B still dumps the pollutant and kills the fish. Hence, factory B is now counterfactually responsible for the death of the fish. It is easy to see that, even in the case of consecutive decisions, nobody is counterfactually responsible for the death under the profiles $(1,0)$ and $(1,1)$ because under these profiles neither of the factories has a strategy that {\em guarantees} that the fish stays alive. As a result, if the decisions are made consecutively, then the responsibility gap is the set
\begin{equation}\label{G1}
G=\{(1,0),(1,1)\}.    
\end{equation}
Note that, in the above example, switching from simultaneous to consecutive decision-making {\em reduces the responsibility gap} from $\{(1,0),(0,1),(1,1)\}$ to $\{(1,0),(1,1)\}$. This is not a coincidence. Consecutive decision-making provides additional information to the parties that decide later. This information might {\em enable} them to prevent a violation of the deontic constraint. If they do not use this ability and the constraint is violated, then they become counterfactually responsible. In many situations, the responsibility gap is an undesirable property of a decision-making mechanism. In such situations, the introduction of an order in which the decisions are made is one of the techniques used by mechanism designers to minimise the responsibility gap. For example, road traffic laws use ``yield'' or ``give way'' signs, the time of arrival to an intersection, and the ``priority to the right'' rule to establish the order in which drivers must make their decisions. This eliminates the responsibility gap and, in the case of an accident, ascribes the responsibility to a specific driver. {\em In the rest of this paper, we only consider decision-making mechanisms where the agents make the decisions in a consecutive order.}

\subsubsection*{Second-Order Responsibility}

As we have seen above, the introduction of the order in which agents act might shrink the responsibility gap, but not be able to eliminate it completely. Under any action profile that belongs to the gap, the deontic constraint is violated and nobody is counterfactually responsible for this. However, there might be an agent who is {\em counterfactually responsible for the gap}. We refer to this responsibility for the gap as {\em second-order responsibility} for the violation of the deontic constraint. In the case when both factories have accumulated 20kg of the pollutant, the gap is given by equation~\eqref{G1}. Note that under both profiles in this set, factory A is second-order responsible for the death of the fish because it has had an action (``do not dump'') that would guarantee that the resulting profile does not belong to set $G$. By not dumping the pollutant, factory A guarantees that someone (factory B in our case) will be counterfactually responsible if the fish dies. Intuitively, second-order responsibility is the responsibility for the fact that no one is accountable.

\subsubsection*{Second-Order Gap}

Let us now consider a setting where there is a third factory, C. Each of the three factories has accumulated 20kg of pollutant and it still takes only 15 kg of the pollutant to kill the fish. 
Thus, the deontic constraint $\gamma$ is $\neg(p_a\vee p_b\vee p_c)$. Suppose that the factories make their decisions in the order: A, B, C. If either factory A or B dumps the pollutant, then the fish dies and no agent has had a strategy that would guarantee that the fish stays alive. Hence, nobody is counterfactually responsible for the death. If only factory C dumps the pollutant, then it has had a strategy (``do not dump'') that would guarantee that the fish stays alive. Hence, in such a situation, factory C is counterfactually responsible for the death. In other words, the responsibility gap is
\begin{equation*}
G_1\!=\!\{(0,1,p_c)\!\mid\! p_c\in\{0,1\}\}
\cup
\{(1,p_b,p_c)\!\mid\! p_b,p_c\in\{0,1\}\}.
\end{equation*}
If the profile belongs to the set $\{(0,1,p_c)\mid p_c\in\{0,1\}\}$, then factory B is second-order responsible because it has a strategy (``do not dump'') that guarantees that the decision-making process will avoid gap $G_1$. This is because this strategy guarantees that if the fish is dead, then factory C is responsible for the death. At the same time, if the profile belongs to the set
\begin{equation}\label{15-sep-a}
G_2=\{(1,p_b,p_c)\mid p_b,p_c\in\{0,1\}\},    
\end{equation}
then no agent has a strategy to guarantee that gap $G_1$ is avoided. Thus, under each profile from set $G_2$, not only the fish is dead and nobody is counterfactually responsible for this, but nobody is second-order responsible for the death either. We say that set $G_2$ is the {\em second-order responsibility gap}. Informally, the second-order responsibility gap consists of all profiles under which the deontic constraint is violated and nobody is responsible for the fact that nobody is accountable for the violation.

As a less formal example, imagine a situation when a paper submitted to a conference has not been reviewed by the PC member it got assigned to. In this case, the PC member is responsible for the failure to review. However, if no PC member has been assigned to review the paper, then there is a responsibility gap and the PC Chair is second-order responsible for the failure to review the paper. If the PC Chair gets ill and nobody has been assigned to replace the Chair, then there is a second-order gap.

\subsubsection*{First Result}
  
Note that, in our three-factory example, factory A has a strategy (``do not dump'') that guarantees that gap $G_2$ is avoided, see equation~\eqref{15-sep-a}. Thus, under any profile that belongs to set $G_2$, factory A is {\em third-order responsible} for the death of the fish. One can also define a third-order gap as a set of all profiles in the second-order gap under which nobody is third-order responsible for a violation of the deontic constraint. In our example, the third-order gap is empty. In this paper, we consider $n$-order responsibility and $n$-order gaps. It is not a coincidence that, in our example with three agents, the third-order gap is empty. In Theorem~\ref{first result theorem}, we show that in consecutive decision-making mechanisms with $n$ agents, the $n$-order gap is always empty. In other words, {\bf if a deontic constraint is violated in an $n$-agent mechanism, then there exists $d\le n$ such that at least one agent is $d$-order responsible for the violation}. The notion of $n$-order responsibility gap in a more general ``extensive form game'' setting is introduced in~\cite{s24aaai}. Shi showed that in that more general setting, the $d$-order gap is empty if $d\ge N-1$, where $N$ is the number of leaf nodes in the extensive form game. When translated to our setting, this means that the gap of order $2^n-1$ is always empty. The reason why our result is so much stronger is that, in our setting, each agent acts just once. In an extensive form game, the same agent might act multiple times. \citet{s24aaai} is also using the term ``higher-order responsibility'' in a sense similar but more general than ours.  

\subsubsection*{Second Result}

Let us now consider a setting where it still takes 15kg of the pollutant to kill the fish, but factories A, B, and C have accumulated 20kg, 10kg, and 10kg of the pollutant, respectively. The deontic constraint is now $\gamma=\neg(p_a\vee(p_b\wedge p_c))$. In such a setting, the fish dies under the following set of action profiles:
$$
\{(1,p_b,p_c)\mid p_b,p_c\in \{0,1\}\}\cup\{(0,1,1)\}.
$$
Out of these profiles, only under the last profile, somebody (actually both, factory B and factory C) is counterfactually responsible for the death of the fish. Thus, the responsibility gap is
$$
G'_1=\{(1,p_b,p_c)\mid p_b,p_c\in \{0,1\}\}.
$$
Note that factory A has a strategy (``do not dump'') to avoid this gap. Thus, factory A is counterfactually responsible for the gap under each profile in set $G'_1$. This means that, in this setting, the second-order responsibility gap is empty. As we have seen in equation~\eqref{15-sep-a}, in other settings such a gap might be nonempty. Thus, in some settings, the $d$-order responsibility gap might be empty for $d$ strictly less than $n$. {\bf Our second result is the complexity analysis of deciding for any given mechanism if the $d$-order responsibility gap is empty. 
In Theorem~\ref{complete theorem}, we show that this problem  is $\Pi_{2d+1}$-complete.} 
\citet{s24aaai} discussed the algorithmic complexity of model-checking of an arbitrary formula that contains counterfactual responsibility modality. The absence of the $d$-order responsibility gap can be expressed by such a formula. Shi proved that model-checking can be done in polynomial time as a function of the number of leaf nodes in the tree representing the mechanism. In the case of the mechanisms discussed in the current paper, the number of leaf nodes is exponential in terms of the number of agents. Thus, when translated into our setting, her result only gives an exponential upper bound on the complexity of deciding if the $d$-order responsibility gap is empty.

\subsubsection*{Group Responsibility}

The main alternative to higher-order responsibility is {\em group responsibility} \cite{nt19aaai,l21pt,ygcsnj21ieee,ygsdjn21aamas,dy23as,ygsdjnr23ais}. As an example, consider our three-factory scenario above: each factory has accumulated 20kg of pollutant and it takes 15kg to kill the fish. Suppose that the first factory (A) does {\em not} dump the pollutant and the other two factories dump (first B, then C). The fish is killed. Note that neither of the three factories had an individual strategy that would guarantee saving the fish. Hence, none of them is counterfactually (first-order) responsible for killing the fish. At the same time, after factory A did not dump the pollutant, factories B and C had an opportunity to {\em coordinate} their actions and agree that neither of them would dump. This is a potential {\em group strategy} of B and C to guarantee that the fish is alive. Of course, the group consisting of all three factories also had a strategy to save the fish. The standard approach to define counterfactual group responsibility is to blame only the {\em minimal} (in terms of subset relation) group that could have prevented~\cite{nt20ai}. In our example, such a group consists of just B and C. Note that group responsibility diffuses the responsibility between B and C, potentially creating a ``circle of blame'' effect.
Note that this approach is different from Necessary Element of a Sufficient Set causality by events or by specific actions~\cite{w85clr}, rather than agents. 

At the same time, in this example, only factory B is individually second-order responsible for the death of the fish. Indeed, if factory B refrains from dumping the pollutant, it will prevent the first-order responsibility gap (if C dumps, C will be counterfactually responsible). As we see from this example, higher-order responsibility avoids the diffusion associated with group responsibility and promotes individual accountability for the outcome.

\subsubsection*{``Responsibility Gap'' in the Literature}

The responsibility gap is also often referred to as ``responsibility void''. These terms are consistently used in the literature to refer to a lack of a responsible agent in many different contexts. An important issue in AI ethics is whether autonomous systems should be treated as having ``moral agency''. If such systems are denied the moral agency status, then there is a commonly acknowledged responsibility gap for the decisions taken by the systems~\cite{m04eit,ct15pt,bhlmmp20ai,c20see,g20eit,sm21pt,t21pt,k22eit,o23pt,hv23synthese}. The same terms are also used in group responsibility discussions when the group is responsible for an outcome but each individual agent is not~\cite{l21pt,ygcsnj21ieee,ygsdjn21aamas,dy23as,ygsdjnr23ais}. Finally, outside moral agency and group responsibility contexts, these terms are used in decision-making mechanism analysis to refer to the outcomes that no agent is individually responsible for~\cite{bh11pq,d18pss,bh18ej,dp22scw,d22}. Our work falls into the third category.

The responsibility gap has been studied in detail for ``discursive dilemma''~\cite{l06ethics} -- a group decision-making mechanism where multiple agents are asked to independently rank available alternatives (say, job candidates) based on several required criteria. There might be a situation where each agent finds the alternative unacceptable due to one of the criteria, but the majority finds that the alternative satisfies all of the criteria. Ascribing individual responsibility for the group choice of the alternative in such a situation is problematic. \cite{bh18ej,dp22scw} studied responsibility gaps in ``discursive dilemma''-like settings. In such a setting, it is almost impossible for an agent to have a strategy that {\em guarantees} that an alternative is avoided. So, in both of these works Frankfurt's ``\dots\ could have done otherwise'' is interpreted as not taking an action which is the {\em most effective} in preventing the outcome. Braham and van Hees show that if a ``discursive dilemma''-like setting has no responsibility gap and does not have what the authors call ``fragmentation'' of responsibility, then the mechanism must be a dictatorship, where a single agent determines the group decision~\cite{bh18ej}. \cite{dp22scw} gives necessary and sufficient conditions for a responsibility gap to exist in a ``discursive dilemma''-like setting.

\section{Formal Setting}

In this section, we describe the formal setting in which we state and prove our technical results. We assume some fixed infinite countable set of propositional variables. By a Boolean formula, we mean any expression built out of propositional variables using conjunction $\wedge$, disjunction $\vee$, and negation $\neg$. By a Boolean formula with quantifiers, we mean a formula that, in addition, can use universal $\forall$ and existential $\exists$ quantifiers over propositional variables. If $\mathbf{v}$ is an ordered set of propositional variables $v_1,\dots,v_k$, then by $\forall\mathbf{v}\phi$ and  $\exists\mathbf{v}\phi$ we mean the formulae $\forall v_1 \dots\forall v_k\phi$ and  $\exists v_1 \dots\exists v_k\phi$, respectively.

\begin{definition}\label{mechanism definition}
A (sequential decision-making) mechanism is a triple $(n,\mathbf{v},\gamma)$, where 
\begin{enumerate}
    \item integer $n\ge 0$ is the number of ``agents'',
    \item $\mathbf{v}=\{\mathbf{v}_i\}_{1\le i\le n}$ is a family of disjoint ordered sets of propositional variables; by $\cup \mathbf{v}$ we denote the unordered union of these sets,
    \item ``deontic constraint'' $\gamma$ is a Boolean formula (without quantifiers) whose variables belong to the set $\cup \mathbf{v}$.
\end{enumerate}
\end{definition}

In our introductory examples, agents were factories A, B, and C. In the formal setting of Definition~\ref{mechanism definition}, agents are numbers 1, \dots, $n$. In those examples, we have assumed that the action of each agent $i$ is represented by a single propositional variable $p_i$. In the more general setting of Definition~\ref{mechanism definition}, an action of each agent $i$ is represented by an ordered set $\mathbf{v}_i$ of Boolean variables. To capture our introductory examples, we assume that $\mathbf{v}_1=\{p_a\}$, $\mathbf{v}_2=\{p_b\}$, and $\mathbf{v}_3=\{p_c\}$.

Next, we will express the $d$-order responsibility of an agent $i$ for the violation of the deontic constraint as a Boolean formula. Let us start with our introductory example where each of the three factories has accumulated 20kg of the pollutant. In this situation, the deontic constraint is
$\gamma=\neg (p_a\vee p_b\vee p_c)$.
Because factory A makes the decision first, for it to be 
 (counterfactually) responsible, (i) the constraint should be violated and (ii) factory A should have an action that guarantees the constraint no matter what the actions of the other agents are. These two conditions can be expressed as
\begin{enumerate}
    \item $p_a\vee p_b\vee p_c$, where  $p_a$, $p_b$, and $p_c$ are unbounded variables representing the actual actions,
    \item $\exists p_a\forall p_b\forall p_c \neg(p_a\vee p_b\vee p_c)$, where $p_a$, $p_b$, and $p_c$ are bounded variables ranging over all possible actions.
\end{enumerate}
Thus, the responsibility of factory A under action profile $(p_a,p_b,p_c)$ can be expressed
by the formula:
\begin{equation}\label{Ra definition}
\R_a=(p_a\vee p_b\vee p_c)\wedge \exists p_a\forall p_b\forall p_c \neg(p_a\vee p_b\vee p_c).    
\end{equation}
Note that it is common in mathematics to avoid using the same variable names for bounded and unbounded variables. To do this, one can re-write formula $\R_a$ in an equivalent form. For example, as 
$(p_a\vee p_b\vee p_c)\wedge \exists x\forall y\forall z \neg(x\vee y\vee z)$. In this paper, we decided to define formula $\R_a$ as given by equation~\eqref{Ra definition} because it makes it significantly simpler to express $d$-order responsibility later. It is important to remember, however, that there is no connection between identically named variables inside and outside of the scope of a corresponding quantifier.

Formula $\R_a$ expresses (counterfactual) responsibility of factory A in the sense that, for any action profile, if the values of unbounded variables $p_a$, $p_b$, and $p_c$ are chosen as specified by the action profile, then the above formula has value 1 if and only if factory A is responsible for the violation of the deontic constraint $\gamma$. 

Recall, that factory B makes its decision after factory A. Thus, for factory B to be (counterfactually) responsible, (i) the deontic constraint should be violated and (ii) factory B should have an action that guarantees the constraint {\bf under the current action of factory A} and any possible action of factory C. This can be expressed as:
\begin{equation}\label{Rb definition}
\R_b=(p_a\vee p_b\vee p_c)\wedge \exists p_b\forall p_c \neg(p_a\vee p_b\vee p_c).    
\end{equation}
Note that both occurrences of variable $p_a$ in the above formula are unbounded. Thus, they both refer to the actual action of factory A. Similarly, the counterfactual responsibility of factory C is expressed by the formula
\begin{equation}\label{Rc definition}
\R_c=(p_a\vee p_b\vee p_c)\wedge \exists p_c \neg(p_a\vee p_b\vee p_c).    
\end{equation}
More generally, under an arbitrary decision-making mechanism $(n,\mathbf{v},\gamma)$, the (counterfactual) responsibility of agent $i\le n$ is expressible by the formula:
\begin{equation}\label{24-sep-a}
\R_i=\neg \gamma \wedge \exists\mathbf{v}_i\forall\mathbf{v}_{i+1}\dots\forall\mathbf{v_n}\gamma.    
\end{equation}
Then, the responsibility gap can be expressed by the formula
$$\G=\neg\gamma \wedge \bigwedge_{j\le n}\neg\R_j.$$
The above formula expressed the gap in the sense that if unbounded occurrences of variables from ordered sets $\mathbf{v}_1,\dots, \mathbf{v}_n$ represent the actions of the agents under some profile, then this formula has Boolean value 1 if and only if the profile belongs to the gap. Recall that the second-order responsibility for the violation of the deontic constraint is the counterfactual responsibility for the gap. Thus, it can be expressed by the formula
\begin{multline*}
\R^2_i= \left(\neg\gamma \wedge \bigwedge_{j\le n}\neg\R_j\right)\\
\wedge
\exists\mathbf{v}_i\forall\mathbf{v}_{i+1}\dots\forall\mathbf{v}_n
\neg \left(\neg\gamma \wedge \bigwedge_{j\le n}\neg\R_j\right).
\end{multline*}
The third-order responsibility is expressible by the formula
\begin{multline*}
\R^3_i= \left(\neg\gamma \wedge \bigwedge_{j\le n}\neg\R_j\wedge \bigwedge_{j\le n}\neg\R^2_j\right)\wedge\\
\exists\mathbf{v}_i\forall\mathbf{v}_{i+1}\dots\forall\mathbf{v}_n
\neg \left(\neg\gamma \wedge \bigwedge_{j\le n}\neg\R_j\wedge \bigwedge_{j\le n}\neg\R^2_j\right).
\end{multline*}
In general, for any order $d>0$, responsibility of agent $i$ of order $d$ is expressible by the formula
\begin{align}
\R^d_i&= \left(\neg\gamma \wedge \bigwedge_{j\le n}\neg\R_j\wedge \bigwedge_{j\le n}\neg\R^2_j\wedge \dots \wedge \bigwedge_{j\le n}\neg\R^{d-1}_j\right)\wedge\nonumber\\
&\exists\mathbf{v}_i\forall\mathbf{v}_{i+1}\dots\forall\mathbf{v}_n \nonumber \\
&\hspace{-0mm}\neg\left(\!\neg\gamma \wedge \bigwedge_{j\le n}\!\neg\R_j\wedge \bigwedge_{j\le n}\!\neg\R^2_j\wedge \dots\wedge \bigwedge_{j\le n}\neg\R^{d-1}_j\!\right).\label{R definition}
\end{align}
Recall that $\mathbf{v}=\{\mathbf{v}_i\}_{i\le n}$ is a family of disjoint ordered sets of propositional variables. By an {\em action profile} we mean any family $\mathbf{s}=\{\mathbf{s}_i\}_{i\le n}$ of ordered sets of Boolean values (0 or 1) such that $|\mathbf{s}_i|=|\mathbf{v}_i|$ for each agent $i\le n$. For any formula $\phi$ whose unbounded propositional variables belong to the set $\cup \mathbf{v}$, let $\phi[\mathbf{s}/\mathbf{v}]$ be the Boolean value (0 or 1) of formula $\phi$ under the valuation that assigns each variable from set $\mathbf{v}_i$ the corresponding Boolean value from set $\mathbf{s}_i$.

\begin{definition}\label{d-order responsibility}
Under an action profile $\mathbf{s}$ of a mechanism $(n,\mathbf{v},\gamma)$, an agent $i\le n$ is $d$-order responsible (for the violation of the deontic constraint)  if $\R^d_i[\mathbf{s}/\mathbf{v}]=1$.   
\end{definition}


\begin{definition}\label{d-gap-free}
A mechanism is $d$-gap-free if, for each action profile that does not satisfy the deontic constraint, there is $d'\le d$ and an agent $i\le n$ who is $d'$-order responsible under the action profile.  
\end{definition}

The set of all $d$-gap-free mechanisms is denoted by $\GF^{d}$.


\begin{theorem}\label{gap-free monotonicity}
$\GF^{d_1}\subseteq \GF^{d_2}$ for each $d_1\le d_2$.    
\end{theorem}
\begin{proof}
By Definition~\ref{d-gap-free} and the assumption $d_1\le d_2$, any $d_1$-gap-free mechanism is a $d_2$-gap-free mechanism. Therefore, $\GF^{d_1}\subseteq \GF^{d_2}$.  
\end{proof}

\section{First Result}

In this section, we prove our first result that any non-trivial mechanism with $n$ agents is $n$-gap-free. This result is stated as Theorem~\ref{first result theorem} at the end of this section. The proof of the next lemma is in the appendix.

\begin{lemma}\label{first result lemma}
For any mechanism $(n,\mathbf{v},\gamma)$, any nonnegative $k\le n$, any action profile $\mathbf{s}$   such that $\gamma[\mathbf{s}/\mathbf{v}]=1$, and any action profile $\mathbf{s}'$ such that
$\mathbf{s}_i=\mathbf{s}'_i$ for each $i\le k$, either $\gamma[\mathbf{s}'/\mathbf{v}]=1$ or there is positive $d\le n-k$ and an agent $i\le n$ who is $d$-order responsible under action profile $\mathbf{s}'$.  
\end{lemma}

\begin{theorem}\label{first result theorem}
If Boolean formula $\gamma$ is satisfiable, then $(n,\mathbf{v},\gamma)\in \GF^n$.    
\end{theorem}
\begin{proof}
The assumption that the Boolean formula $\gamma$ is satisfiable implies that there is an action profile $\mathbf{s}$ such that $\gamma[\mathbf{s}/\mathbf{v}]=1$.  Consider an action profile $\mathbf{s}'$ such that $\gamma[\mathbf{s}'/\mathbf{v}]=0$. By Definition~\ref{d-gap-free}, it suffices to show that there is $d\le n$ and an agent $i\le n$ who is $d$-order responsible under action profile $\mathbf{s}'$. The last statement follows from Lemma~\ref{first result lemma} in the case $k=0$.
\end{proof}
Together with Theorem~\ref{gap-free monotonicity}, the above result implies that if $\gamma$ is satisfiable, then $(n,\mathbf{v},\gamma)\in \GF^d$ for each $d\ge n$.

\section{Second Result}

In this section, we prove that set $\GF^d$ is $\Pi_{2d+1}$-complete. The result is stated as Theorem~\ref{complete theorem} at the end of the section.

\begin{lemma}\label{in Pi 2d+1}
For each $d\ge 1$, set ${\sf GF}^d$ belongs to the class $\Pi_{2d+1}$.   
\end{lemma}
\begin{proof} First, observe that the following claim holds.
\begin{claim}\label{24-sep-b}
For each $d\ge 1$ and each $i\le n$, formula $\R^d_i$ belongs to class $\Sigma_{2d}$.   
\end{claim}
\begin{proof-of-claim}
We prove the claim by induction on $d$. 
If $d=1$, then $\R^d_i$ is the formula 
$\neg \gamma \wedge \exists\mathbf{v}_i\forall\mathbf{v}_{i+1}\dots\forall\mathbf{v_n}\gamma$
by equation~\eqref{R definition} (or equation~\eqref{24-sep-a}).
Note that $\gamma$ is a Boolean formula without quantifiers by item~3 of Definition~\ref{mechanism definition}.
Hence, formula $\R^d_i$ belongs to class $\Sigma_2$. 
The induction step follows from equation~\eqref{R definition} and the induction hypothesis.
\end{proof-of-claim}
By Definition~\ref{d-gap-free} and Definition~\ref{d-order responsibility}, a mechanism $(n,\mathbf{v},\gamma)$ is $d$-gap-free if and only if the following closed formula is a tautology: 
$$\forall\mathbf{v}_1\dots\forall\mathbf{v}_n \left(\gamma\vee\bigvee_{d'\le d}\bigvee_{j\le n}\R^{d'}_{j}\right).$$
By Claim~\ref{24-sep-b}, this formula belongs to class $\Pi_{2d+1}$, which implies the statement of the lemma.
\end{proof}

Let us now prove that the set $\GF^d$ is $\Pi_{2d+1}$-complete by defining a polynomial reduction of $\Pi_{2d+1}$ problems to $\GF^d$. Towards this goal, for each Boolean formula, we specify a decision-making mechanism.

\begin{definition}\label{canonical mechanism}
For any $d\ge 0$, any Boolean formula $\phi$ without quantifiers, and any family of disjoint ordered sets $\mathbf{x}_1,\dots,\mathbf{x}_{2d+1}$ of propositional variables whose union is equal to the set of all propositional variables in formula $\phi$, mechanism $M(\phi,\mathbf{x}_1,\dots,\mathbf{x}_{2d+1})=(2d+1,\mathbf{v},\gamma)$ is defined as follows:
\begin{enumerate}
    \item $\mathbf{v}_{2i+1}=\mathbf{x}_{2i+1}$ for each $i$ such that $0\le i\le d$,
    \item $\mathbf{v}_{2i}$ is equal to the concatenation of a new single variable $q_{2i}$ to the end of the ordered list of variables $\mathbf{x}_{2i}$, for each $i$ such that $1\le i\le d$,
    \item $\gamma=\phi\wedge\bigwedge_{1\le i\le d}q_{2i}$.
\end{enumerate}
\end{definition}
As an example, suppose that $d=1$, formula $\phi$ is the formula $x_1\vee x_2\vee x_3$, $\mathbf{x}_1=\{x_1\}$, $\mathbf{x}_2=\{x_2\}$, and $\mathbf{x}_3=\{x_3\}$. Then, mechanism $M(\phi,\mathbf{x}_1,\mathbf{x}_2,\mathbf{x}_{3})$ has three agents (called agent 1, agent 2, and agent 3). The action of agent 1 is defined by the singleton set of variables  $\mathbf{v}_1=\mathbf{x}_1=\{x_1\}$. The action of agent 2 is defined by the ordered set  $\mathbf{v}_2=\{x_2,q_2\}$, where $q_2$ is a new propositional variable that does not appear among variables $x_1$, $x_2$, and $x_3$. The action of agent 3 is defined by $\mathbf{v}_3=\mathbf{x}_3=\{x_3\}$. The deontic constraint $\gamma$ is $(x_1\vee x_2\vee x_3)\wedge q_2$.

The connection between formula $\phi$ and the decision-making mechanism $M(\phi,\mathbf{x}_1,\dots,\mathbf{x}_{2d+1})$ is stated in Lemma~\ref{left to right lemma} and Lemma~\ref{right to left lemma} below. To prove both of these lemmas, we introduce a {\em two-player extensive form game} between Devil and Moralist. In the game, Devil is in charge of odd-numbered sets $\mathbf{v}_1,\mathbf{v}_3,\mathbf{v}_5,\dots,\mathbf{v}_{2d+1}$ and Moralist is in charge of even-numbered sets $\mathbf{v}_2,\mathbf{v}_4,\dots,\mathbf{v}_{2d}$. The game consists of Devil and Moralist consequently choosing the values of variables in ordered sets they control: first, Devil chooses the values of all variables in $\mathbf{v}_1$, then Moralist chooses the values of all variables in $\mathbf{v}_2$, then Devil chooses $\mathbf{v}_3$, then Moralist chooses $\mathbf{v}_4$, and so on. Devil always concludes the game by choosing the values of all variables in set $\mathbf{v}_{2d+1}$. Once the game is over, they have specified an action profile of the mechanism  $M(\phi,\mathbf{x}_1,\dots,\mathbf{x}_{2d+1})$. 
The Devil's objective is to violate the deontic constraint $\gamma$ on the action profile they specified. The objective of Moralist is to satisfy~$\gamma$.

Figure~\ref{tree figure} shows the game between Devil and Moralist for the example we discussed above. The game starts at a root node $\epsilon$ with none of the variables specified. Then, Devil can choose the value of $x_1$ to be either 0 (go left) or 1 (go right). Then, Moralist chooses the values of $x_2$ and $q_2$. This corresponds to a transition from node $x_1$ to node $x_1x_2q_2$ in the tree. Finally, Devil picks the value of $x_3$, which completely specifies the action profile $x_1x_2q_2x_3$ of the decision-making mechanism. Note that, to keep the drawing clean, in Figure~\ref{tree figure} we write $x_1x_2q_2x_3$ instead of more accurate $(x_1,(x_2,q_2),x_3)$. In the figure, we labelled with $\gamma$ the leaf nodes whose action profiles satisfy the formula $\gamma=(x_1\vee x_2\vee x_3)\wedge q_2$. These are the outcomes of the game in which Moralist prevails.

It is easy to see that in our example in Figure~\ref{tree figure}, Moralist has a {\em strategy} to guarantee $\gamma$. By Win(Moralist) we denote the set of all (not necessarily leaf)  nodes in which Moralist has a {\em strategy} to guarantee the deontic constraint. For the example depicted in Figure~\ref{tree figure},
\begin{multline*}
\text{Win(Moralist)}=
\{\epsilon, 0, 1, 011, 101, 111,0011,\\0110,0111,1010,1011,1110, 1111\}.    
\end{multline*}

\begin{figure*}
\begin{center}
\scalebox{.5}{\includegraphics{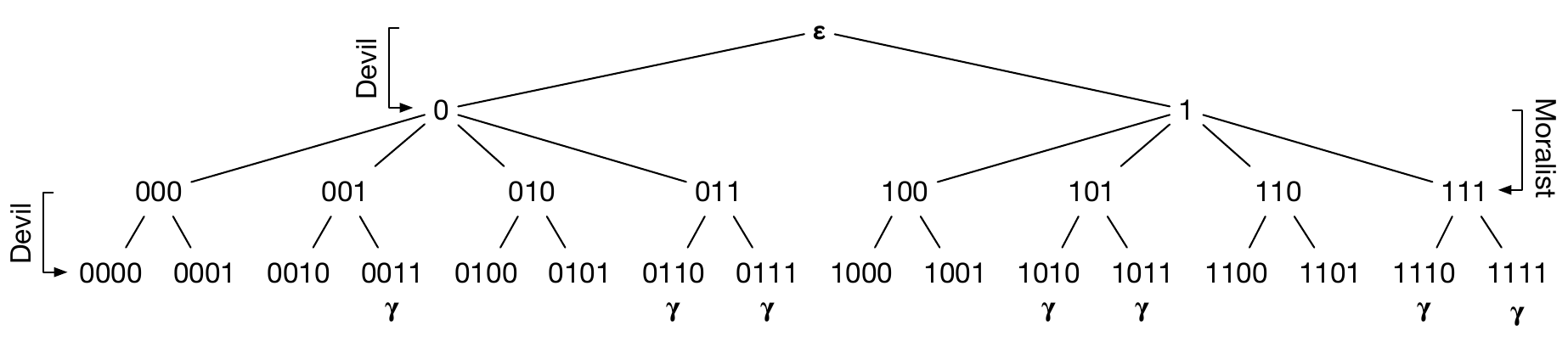}}
\caption{Extensive form game between Devil and Moralist.}\label{tree figure}
\end{center}
\end{figure*}

In Lemma~\ref{lemma D}, we establish a connection between the responsibility of agents under the profiles of mechanism $M(\phi,\mathbf{x}_1,\dots,\mathbf{x}_{2d+1})$ and existence of winning strategies in the two-player game between Devil and Moralist. In preparation for proving Lemma~\ref{lemma D}, below we state three small observations about winning strategies in the game.

\begin{lemma}\label{lemma A}
If $(\mathbf{s}_1,\dots, \mathbf{s}_{2k})\in \text{\em Win(Moralist)}$, then $(\mathbf{s}_1,\dots, \mathbf{s}_{2k},\mathbf{s}_{2k+1})\in \text{\em Win(Moralist)}$. 
\end{lemma}
\begin{proof}
According to the rules of the two-player game, it is Devil's turn to make a move at the node $(\mathbf{s}_1,\dots, \mathbf{s}_{2k})$. Thus, in order for Moralist to have a winning strategy at this node, Moralist must have such a strategy after each possible move of Devil. In particular, Moralist must have a winning strategy at the node $(\mathbf{s}_1,\dots, \mathbf{s}_{2k},\mathbf{s}_{2k+1})$.  
\end{proof}

The proofs of the next three lemmas are in the appendix.

\begin{lemma}\label{lemma B}
If $(\mathbf{s}_1,\dots, \mathbf{s}_{2k+1})\in \text{\em Win(Moralist)}$, then there exists an action $\mathbf{s}'_{2k+2}$ of agent $2k+2$ such that $(\mathbf{s}_1,\dots, \mathbf{s}_{2k+1},\mathbf{s}'_{2k+2})\in \text{\em Win(Moralist)}$.    
\end{lemma}

\begin{lemma}\label{lemma C}
If $\mathbf{s}=(\mathbf{s}_1,\dots, \mathbf{s}_{2d+1})\in \text{\em Win(Moralist)}$, then
$\gamma[\mathbf{s}/\mathbf{v}]=1$.
\end{lemma}

\begin{lemma}\label{lemma D}
For any integer $k$ such that $0\le k\le d$ and any action profile $\mathbf{s}=(\mathbf{s}_1,\dots,\mathbf{s}_{2d+1})$, if $(\mathbf{s}_1,\dots,\mathbf{s}_{2k+1})\in \text{\em Win(Moralist)}$,
then either $\gamma[\mathbf{s}/\mathbf{v}]=1$ or there is positive $d'\le d-k$ and an agent $i\le 2d+1$ who is $d'$-order responsible under action profile $\mathbf{s}$.
\end{lemma}

\begin{lemma}\label{left to right lemma}
If the closed Boolean formula with quantifiers $\forall \mathbf{x}_1\exists\mathbf{x}_2\forall\mathbf{x}_3\dots \forall\mathbf{x}_{2d+1}\phi$ is true, then, for any action profile $\mathbf{s}$, either $\gamma[\mathbf{s}/\mathbf{v}]=1$ or there is positive $d'\le d$ and an agent $i\le 2d+1$ who is $d'$-order responsible under profile~$\mathbf{s}$.   
\end{lemma}
\begin{proof}
Suppose that the closed formula $\forall \mathbf{x}_1\exists\mathbf{x}_2\forall\mathbf{x}_3\dots \forall\mathbf{x}_{2d+1}\phi$ is true. Thus, the formula
\begin{multline*}
\forall\mathbf{x}_1\exists\mathbf{x}_2\exists{q_2}\forall\mathbf{x}_3\exists\mathbf{x}_4\exists{q_4}\dots \exists\mathbf{x}_{2d}\exists q_{2d}\forall\mathbf{x}_{2d+1}\\(\phi\wedge q_2\wedge q_4 \wedge \dots \wedge q_{2d})    
\end{multline*}
is also true because the values of $q_2,q_4,\dots,q_{2d}$ could be chosen to be 1.
Then, by Definition~\ref{canonical mechanism}, the formula 
$\forall \mathbf{v}_1\exists\mathbf{v}_2\forall \mathbf{v}_3\dots \forall\mathbf{v}_{2d+1}\gamma$ is true. 
Recall that Moralist is in charge of even-numbered sets $\mathbf{v}_2,\mathbf{v}_4,\dots,\mathbf{v}_{2d}$.
Thus, at the initial moment in the game, Moralist has a strategy to guarantee that the deontic constraint $\gamma$ is satisfied. In other words, $\epsilon\in \text{Win(Moralist)}$. Suppose $\mathbf{s}=(\mathbf{s}_1,\dots,\mathbf{s}_{2d+1})$.
Then, $(\mathbf{s}_1)\in \text{Win(Moralist)}$ by Lemma~\ref{lemma A}. Therefore, by Lemma~\ref{lemma D} for $k=0$, either $\gamma[\mathbf{s}/\mathbf{v}]=1$ or there is $d'\le d$ and an agent $i\le 2d+1$ who is $d'$-order responsible under action profile $\mathbf{s}$.
\end{proof}

\begin{lemma}\label{right to left lemma}
If the closed Boolean formula with quantifiers $\forall \mathbf{x}_1\exists\mathbf{x}_2\forall\mathbf{x}_3\dots \forall\mathbf{x}_{2d+1}\phi$ is false, then there exists an action profile $\mathbf{s}$ such that $\gamma[\mathbf{s}/\mathbf{v}]=0$ and there is no $d'\le d$ such that at least one of the agents is $d'$-order responsible under profile $\mathbf{s}$.  
\end{lemma}
\begin{proof}
The assumption that the closed formula $\forall \mathbf{x}_1\exists\mathbf{x}_2\forall\mathbf{x}_3\dots \forall\mathbf{x}_{2d+1}\phi$ is false implies that the formula
$\exists \mathbf{x}_1\forall\mathbf{x}_2\exists \mathbf{x}_3\dots \exists\mathbf{x}_{2d+1}\neg\phi$ is true. Thus, the formula
\begin{multline*}
\exists\mathbf{x}_1\forall\mathbf{x}_2\forall{q_2}\exists\mathbf{x}_3\forall\mathbf{x}_4\forall{q_4}\dots \forall\mathbf{x}_{2d}\forall q_{2d}\exists\mathbf{x}_{2d+1}\\
\neg (\phi\wedge q_2\wedge q_4 \wedge \dots \wedge q_{2d})
\end{multline*}
is also true because to make a conjunction false it suffices to guarantee that the first conjunct is false. Then, by Definition~\ref{canonical mechanism}, the formula 
$\exists \mathbf{v}_1\forall\mathbf{v}_2\exists \mathbf{v}_3\dots \exists\mathbf{v}_{2d+1}\neg\gamma$ is true. Recall that Devil is in charge of odd-numbered sets $\mathbf{v}_1,\mathbf{v}_3,\dots,\mathbf{v}_{2d+1}$.
Thus, at the initial moment in the game, Devil has (at least one) strategy to guarantee that the deontic constraint $\gamma$ is violated. Let $\sigma$ be one of such strategies.

For any action profile  $\mathbf{s}=(\mathbf{s}_1,\dots, \mathbf{s}_{2d+1})$, we define the ``degree of immorality'' $\|\mathbf{s}\|$ of profile $\mathbf{s}$ as the total number of ``sins'' committed under profile $\mathbf{s}$. The ``sins'' are defined as follows:
\begin{enumerate}
    \item each $i\le 2d+1$ such that $q_{2i}[\mathbf{s}/\mathbf{v}]=0$ is a ``sin'',
    \item it is a ``sin'' if Devil follows strategy $\sigma$ under the profile $\mathbf{s}_1,\dots, \mathbf{s}_{2d+1}$.
\end{enumerate}
Note that $\gamma=\phi\wedge\bigwedge_{1\le i\le d}q_{2i}$ by Definition~\ref{canonical mechanism}. Thus, committing even a single ``sin'' under profile~$\mathbf{s}$ guarantees that $\gamma[\mathbf{s}/\mathbf{v}]=0$. 
Committing two sins ``double-proofs'' it. Committing three sins ``triple-proofs'' it. Intuitively, the ``degree of immorality'' $\|\mathbf{s}\|$ is the measurement of how ``evil'' is the behaviour of the agents 1, \dots, 2d+1 under a profile $\mathbf{s}$.

\begin{claim}\label{zero degree claim}
For any profile $\mathbf{s}$, if $\gamma[\mathbf{s}/\mathbf{v}]=1$, then $\|\mathbf{s}\|=0$.   
\end{claim}
\begin{proof-of-claim}
The statement of the claim holds because committing even a single ``sin'' guarantees that $\gamma[\mathbf{s}/\mathbf{v}]=0$.   
\end{proof-of-claim}



\begin{claim}\label{new claim}
For any action profiles $\mathbf{s}=(\mathbf{s}_1,\dots, \mathbf{s}_{2d+1})$, any $i\le 2d+1$ and any $\mathbf{s}'_i$, there is an action profile $\mathbf{s}'=(\mathbf{s}_1,\dots,\mathbf{s}_{i-1},\mathbf{s}'_i,\mathbf{s}'_{i+1},\dots ,\mathbf{s}'_{2d+1})$ such that 
$
\|\mathbf{s}\|-1\le 
\|\mathbf{s}'\|.
$
\end{claim}
\begin{proof-of-claim}
If number $i$ is odd, then let $\mathbf{s}'$ be the action profile  $(\mathbf{s}_1,\dots,\mathbf{s}_{i-1},\mathbf{s}'_i,\mathbf{s}_{i+1},\dots ,\mathbf{s}_{2d+1})$. Note that $q_{2j}[\mathbf{s}/\mathbf{v}]=q_{2j}[\mathbf{s}'/\mathbf{v}]$ for each $j\le d$ because $i$ is odd. Therefore,
$
\|\mathbf{s}\|-1\le 
\|\mathbf{s}'\|
$ by the definition of the degree of immortality.

If $i$ is even, then Devil is {\em not} in charge of set $\mathbf{v}_i$. Thus, there is an action profile 
$$
(\mathbf{s}_1,..,\mathbf{s}_{i-1},\mathbf{s}'_i,\mathbf{s}'_{i+1},\mathbf{s}_{i+2},\mathbf{s}'_{i+3},\mathbf{s}_{i+4},\mathbf{s}'_{i+5},..,\mathbf{s}_{2d},\mathbf{s}'_{2d+1})$$
under which Devil still follows strategy $\sigma$ if Devil followed $\sigma$ under profile $\|\mathbf{s}\|$. Let $\mathbf{s}'$ be such a profile. Then, 
$
\|\mathbf{s}\|-1\le 
\|\mathbf{s}'\|
$
because $q_{2j}[\mathbf{s}/\mathbf{v}]=q_{2j}[\mathbf{s}'/\mathbf{v}]$ when $2j\neq i$. 
\end{proof-of-claim}

\begin{claim}\label{new strong claim}
For  any profile $\mathbf{s}$, any positive integer $k<\|\mathbf{s}\|$, and any agent $i\le 2d+1$, agent $i$ is not $k$-order responsible under action profile  $\mathbf{s}$. 
\end{claim}
\begin{proof-of-claim}
We prove the statement by induction on $\|\mathbf{s}\|$. If $\|\mathbf{s}\|\le 1$, then the statement vacuously holds because there are no positive integers $k$ such that $k<1$. 

Suppose that there is an integer $k$ such that 
\begin{equation}\label{19-oct-b}
1<k<\|\mathbf{s}\|    
\end{equation}
and a $k$-order responsible agent $i\le 2d+1$ under $\mathbf{s}$. Thus,  
$\R^k_i[\mathbf{s}/\mathbf{v}]=1$ by Definition~\ref{d-order responsibility}. Hence, by equation~\eqref{R definition},
\begin{align*}
&\exists\mathbf{v}_i\forall\mathbf{v}_{i+1}\dots\forall\mathbf{v}_{2d+1} \nonumber \\
&\hspace{-0mm}\neg\left(\!\neg\gamma \wedge \bigwedge_{j\le 2d+1}\!\neg\R_j\wedge \dots\wedge \bigwedge_{j\le 2d+1}\neg\R^{k-1}_j\!\right)[\mathbf{s}/\mathbf{v}]=1.
\end{align*}
Let $\mathbf{s}=(\mathbf{s}_1,\dots, \mathbf{s}_{2d+1})$. Then, there is $\mathbf{s}'_i$ such that
\begin{align*}
\neg\left(\!\neg\gamma \wedge \bigwedge_{j\le 2d+1}\!\neg\R_j \dots\wedge \bigwedge_{j\le 2d+1}\neg\R^{k-1}_j\!\right)[\mathbf{s}'/\mathbf{v}]=1
\end{align*}
for {\em each} profile $\mathbf{s}'=(\mathbf{s}_1,\dots,\mathbf{s}_{i-1},\mathbf{s}'_i,\mathbf{s}'_{i+1},\dots ,\mathbf{s}'_{2d+1})$.
By Claim~\ref{new claim}, profile  $\mathbf{s}'$ can be {\em chosen} in such a way that
\begin{equation}\label{new eq}
\|\mathbf{s}\|-1\le 
\|\mathbf{s}'\|. 
\end{equation}
Then,
\begin{equation}\label{19-oct-a}
\left(\!\gamma \vee \bigvee_{j\le 2d+1}\!\R_j\vee \dots\vee \bigvee_{j\le 2d+1}\R^{k-1}_j\!\right)[\mathbf{s}'/\mathbf{v}]=1.
\end{equation}
Inequalities~\eqref{19-oct-b} and \eqref{new eq} imply that $\|\mathbf{s}'\|\ge k> 1$. Thus, $\gamma[\mathbf{s}'/\mathbf{v}]=0$ by Claim~\ref{zero degree claim}. Hence, by equation~\eqref{19-oct-a}, there is $j\le 2d+1$ and $k'\le k-1$ such that
$\R^{k'}_j[\mathbf{s}'/\mathbf{v}]=1$. Then, by the induction hypothesis,
$\|\mathbf{s}'\|<k'$. Thus, by inequality~\eqref{new eq},
$$\|\mathbf{s}\|\le\|\mathbf{s}'\|+1<k'+1\le (k-1)+1=k,$$  
which contradicts inequality~\eqref{19-oct-b}.
\end{proof-of-claim}

Consider any action profile $\mathbf{s}$ under which Devil follows strategy $\sigma$ and $q_{2i}[\mathbf{s}/\mathbf{v}]=0$ for each $i\le d$. In other words, we consider an action profile under which all $d+1$ possible ``sin''s have been committed. Thus,
\begin{equation}\label{2-oct-a}
\|\mathbf{s}\|=d+1.    
\end{equation}
Note that $\gamma[\mathbf{s}/\mathbf{v}]=0$, because Devil's strategy $\sigma$ guarantees that the deontic constraint is violated.
To finish the proof of the lemma, suppose that there is $d'\le d$ such that at least one agent is $d'$-order responsible under $\mathbf{s}$. Thus, $d'\ge \|\mathbf{s}\|$ by Claim~\ref{new strong claim}. Therefore, $d\ge d'\ge \|\mathbf{s}\|=d+1$ by statement~\eqref{2-oct-a}, which is a contradiction.
\end{proof}
The next result follows from the two previous lemmas.

\begin{lemma}\label{Pi 2d+1 hard}
Set ${\sf GF}^d$ is $\Pi_{2d+1}$-hard.    
\end{lemma}

\begin{theorem}\label{complete theorem}
Set ${\sf GF}^d$ is $\Pi_{2d+1}$-complete.    
\end{theorem}
\begin{proof}
The statement of the theorem follows from Lemma~\ref{in Pi 2d+1} and Lemma~\ref{Pi 2d+1 hard}.  
\end{proof}

\section{Conclusion}

In this paper, we have studied higher-order responsibility gaps in sequential decision-making mechanisms. We have shown that the set of $d$-order gap-free mechanisms is $\Pi_{2d+1}$-complete. At the same time, the mechanism is guaranteed to be $d$-order gap-free when the number of agents is at most $d$. 

\section{Acknowledgments}

The authors acknowledge the support of the project of the National Social Science Fund of China (No.25BZX072). We also would like to acknowledge the contribution of the AI Reviewer who discovered a non-trivial gap in the proof of Lemma~\ref{right to left lemma} that we have been able to fix.

\bibliography{naumov}

\clearpage

\addtolength{\oddsidemargin}{2.5cm}
	\addtolength{\evensidemargin}{2.5cm}
	\addtolength{\textwidth}{-5cm}

	\addtolength{\topmargin}{2.5cm}
	\addtolength{\textheight}{-5cm}


\onecolumn

\begin{center}
    {\LARGE\sc Technical Appendix}

\vspace{3mm}

    {\bf This appendix is not a part of the AAAI-26 proceedings.}

\vspace{2mm}
\end{center}

\noindent{\bf Lemma~\ref{first result lemma}} {\em
For any mechanism $(n,\mathbf{v},\gamma)$, any nonnegative $k\le n$, any action profile $\mathbf{s}$   such that $\gamma[\mathbf{s}/\mathbf{v}]=1$, and any action profile $\mathbf{s}'$ such that
$\mathbf{s}_i=\mathbf{s}'_i$ for each $i\le k$, either $\gamma[\mathbf{s}'/\mathbf{v}]=1$ or there is positive $d\le n-k$ and an agent $i\le n$ who is $d$-order responsible under action profile $\mathbf{s}'$.  
}

\begin{proof} We prove the statement of the lemma by backward induction on $k$. If $k=n$, then $\mathbf{s}=\mathbf{s}'$ by the assumption that $\mathbf{s}_i=\mathbf{s}'_i$ for each $i\le k$. Thus, $\gamma[\mathbf{s}'/\mathbf{v}]=1$ by the assumption  $\gamma[\mathbf{s}/\mathbf{v}]=1$. 

Suppose that $k<n$. 
By the induction hypothesis and Definition~\ref{d-order responsibility}, 
$$\left(\gamma \vee \bigvee_{j\le n}\R_j\vee \bigvee_{j\le n}\R^2_j\vee \dots \vee \bigvee_{j\le n}\R^{n-(k+1)}_j\right)[\mathbf{s}''/\mathbf{v}]=1$$
for each action profile $\mathbf{s}''$ such that $\mathbf{s}_i=\mathbf{s}''_i$ for each $i\le k+1$. Then,
\begin{equation*}
\left(\forall\mathbf{v}_{k+2}\dots\forall\mathbf{v}_{n}\left(\gamma \vee \bigvee_{j\le n}\R_j\vee \bigvee_{j\le n}\R^2_j\vee\dots \vee \bigvee_{j\le n}\R^{n-(k+1)}_j\right)\right)[\mathbf{s}/\mathbf{v}]=1.
\end{equation*}
Thus,
\begin{multline*}
\left(\exists \mathbf{v}_{k+1}\forall\mathbf{v}_{k+2}\dots\forall\mathbf{v}_{n}\left(\gamma \vee \bigvee_{j\le n}\R_j\vee \bigvee_{j\le n}\R^2_j\vee\right.\right. 
\left.\left.\dots \vee \bigvee_{j\le n}\R^{n-(k+1)}_j\right)\right)[\mathbf{s}/\mathbf{v}]=1.
\end{multline*}
Then,
\begin{multline*}
\left(\exists \mathbf{v}_{k+1}\forall\mathbf{v}_{k+2}\dots\forall\mathbf{v}_{n}\left(\gamma \vee \bigvee_{j\le n}\R_j\vee \bigvee_{j\le n}\R^2_j\vee\right.\right.
\left.\left.\dots \vee \bigvee_{j\le n}\R^{n-(k+1)}_j\right)\right)[\mathbf{s}'/\mathbf{v}]=1
\end{multline*}
by the assumption  of the lemma that $\mathbf{s}_i=\mathbf{s}'_i$ for each $i\le k$. Hence, by De Morgan laws,
\begin{multline}\label{21-sep-a}
\left(\exists \mathbf{v}_{k+1}\forall\mathbf{v}_{k+2}\dots\forall\mathbf{v}_{n}\neg\left(\neg\gamma \wedge \bigwedge_{j\le n}\neg\R_j\wedge \bigwedge_{j\le n}\neg\R^2_j\wedge\right.\right.\\ 
\left.\left.\dots \wedge \bigwedge_{j\le n}\neg\R^{n-(k+1)}_j\right)\right)[\mathbf{s}'/\mathbf{v}]=1.    
\end{multline}

If agent $k+1$ is $(n-k)$-order responsible under action profile $\mathbf{s}'$, then the conclusion of the lemma is true. Suppose otherwise. Thus, $\R^{n-k}_{k+1}[\mathbf{s}'/\mathbf{v}]=0$ by Definition~\ref{d-order responsibility}. Hence, by equations~\eqref{R definition} and~\eqref{21-sep-a}, 
\begin{equation*}
\left(\neg\gamma \wedge \bigwedge_{j\le n}\neg\R_j\wedge \bigwedge_{j\le n}\neg\R^2_j\wedge \dots\wedge \bigwedge_{j\le n}\neg\R^{n-(k+1)}_j\right)[\mathbf{s}'/\mathbf{v}]=0.
\end{equation*}
Then, again, by De Morgan's laws,
$$\left(\gamma \vee \bigvee_{j\le n}\R_j\vee \bigvee_{j\le n}\R^2_j\vee \dots \vee \bigvee_{j\le n}\R^{n-(k+1)}_j\right)[\mathbf{s}'/\mathbf{v}]=1.$$
Thus, either $\gamma[\mathbf{s}'/\mathbf{v}]=1$ or there is $d\le n-(k+1)<n-k$ and an agent $j\le n$ such that $\R^{d}_j[\mathbf{s}'/\mathbf{v}]=1$. Therefore, either $\gamma[\mathbf{s}'/\mathbf{v}]=1$ or agent $j$ is $d$-order responsible under profile $\mathbf{s}'$ by Definition~\ref{d-order responsibility}. 
\end{proof}

\noindent{\bf Lemma~\ref{lemma B}} {\em
If $(\mathbf{s}_1,\dots, \mathbf{s}_{2k+1})\in \text{\em Win(Moralist)}$, then there exists an action $\mathbf{s}'_{2k+2}$ of agent $2k+2$ such that $(\mathbf{s}_1,\dots, \mathbf{s}_{2k+1},\mathbf{s}'_{2k+2})\in \text{\em Win(Moralist)}$.    
}
\begin{proof}
According to the rules of the two-player game, it is Moralist's turn to make a move at the node $(\mathbf{s}_1,\dots, \mathbf{s}_{2k+1})$. Thus, in order for Moralist to have a winning strategy at this node, Moralist must have at least one move that transitions the game into another node winning for Moralist. 
\end{proof}

\noindent{\bf Lemma~\ref{lemma C}} {\em
If $\mathbf{s}=(\mathbf{s}_1,\dots, \mathbf{s}_{2d+1})\in \text{\em Win(Moralist)}$, then
$\gamma[\mathbf{s}/\mathbf{v}]=1$.
}
\begin{proof}
The node $(\mathbf{s}_1,\dots, \mathbf{s}_{2d+1})$ is a final (leaf) node in the extensive form game tree. Neither Davil nor Moralist can make a move from this node. Thus, in order for the node to be a winning node for Moralist, the action profile $(\mathbf{s}_1,\dots, \mathbf{s}_{2d+1})$ must satisfy the deontic constraint $\gamma$.  
\end{proof}

\noindent{\bf Lemma~\ref{lemma D}} {\em
For any integer $k$ such that $0\le k\le d$ and any action profile $\mathbf{s}=(\mathbf{s}_1,\dots,\mathbf{s}_{2d+1})$, if $(\mathbf{s}_1,\dots,\mathbf{s}_{2k+1})\in \text{\em Win(Moralist)}$,
then either $\gamma[\mathbf{s}/\mathbf{v}]=1$ or there is positive $d'\le d-k$ and an agent $i\le 2d+1$ who is $d'$-order responsible under action profile $\mathbf{s}$.
}

\begin{proof}
We prove the statement of the lemma by backward induction on $k$. In the base case, $k=d$, the statement of the lemma follows from Lemma~\ref{lemma C}.

Assume that $k<d$. Suppose that $(\mathbf{s}_1,\dots,\mathbf{s}_{2k+1})\in \text{ Win(Moralist)}$. Then, by Lemma~\ref{lemma B}, there is an action $\mathbf{s}'_{2k+2}$ of agent $2k+2$ such that 
\begin{equation}\label{29-sep-a}
(\mathbf{s}_1,\dots,\mathbf{s}_{2k+1},\mathbf{s}'_{2k+2})\in \text{ Win(Moralist)}.    
\end{equation}
Consider arbitrary actions $\mathbf{s}'_{2k+3},\dots,\mathbf{s}'_{2d+1}$ of agents $2k+3,\dots, 2d+1$. Let 
\begin{equation}\label{29-sep-b}
\mathbf{s}'=(\mathbf{s}_1,\dots,\mathbf{s}_{2k+1},\mathbf{s}'_{2k+2},\mathbf{s}'_{2k+3},\dots,\mathbf{s}'_{2d+1}).    
\end{equation}
Note that $(\mathbf{s}_1,\dots,\mathbf{s}_{2k+1},\mathbf{s}'_{2k+2},\mathbf{s}'_{2k+3})\in \text{ Win(Moralist)}$ by Lemma~\ref{lemma A} and statement~\eqref{29-sep-a}. Thus, by the induction hypothesis applied to profile $\mathbf{s}'$ and Definition~\ref{d-order responsibility},
\begin{equation*}
    \left(\gamma\vee \bigvee_{i\le 2d+1}\R_i\vee  
    \bigvee_{i\le 2d+1}\R^2_i\vee\right.
    \left.\dots \vee
    \bigvee_{i\le 2d+1}\R^{d-(k+1)}_i\right)[\mathbf{s}'/\mathbf{v}]=1.
\end{equation*}
Recall that $\mathbf{s}'_{2k+3},\dots,\mathbf{s}'_{2d+1}$ are arbitrary actions. Thus,
\begin{multline*}
    \left(\forall \mathbf{v}_{2k+3}\dots
    \forall \mathbf{v}_{2d+1}
    \left(\gamma\vee \bigvee_{i\le 2d+1}\R_i\vee  
    \bigvee_{i\le 2d+1}\R^2_i\vee\right.\right.\\
    \left.\left.\dots \vee
    \bigvee_{i\le 2d+1}\R^{d-(k+1)}_i\right)\right)[\mathbf{s}'/\mathbf{v}]=1.
\end{multline*}
Hence,
\begin{multline*}
    \left(\exists \mathbf{v}_{2k+2}
    \forall \mathbf{v}_{2k+3}\dots
    \forall \mathbf{v}_{2d+1}
    \left(\gamma\vee \bigvee_{i\le 2d+1}\R_i\vee\right.\right.\\  
    \bigvee_{i\le 2d+1}\R^2_i\vee
    \left.\left.\dots \vee
    \bigvee_{i\le 2d+1}\R^{(d-k)-1}_i\right)\right)[\mathbf{s}'/\mathbf{v}]=1.
\end{multline*}
Note that, in the above Boolean formula with quantifiers, the free variables are only those that belong to the ordered sets $\mathbf{v}_1,\dots,\mathbf{v}_{2k+1}$. Also, that profiles $\mathbf{s}$ and $\mathbf{s}'$ assign these variables the same values due to equation~\eqref{29-sep-b}. Then,
\begin{multline*}
    \left(\exists \mathbf{v}_{2k+2}
    \forall \mathbf{v}_{2k+3}\dots
    \forall \mathbf{v}_{2d+1}
    \left(\gamma\vee \bigvee_{i\le 2d+1}\R_i\vee\right.\right.\\  
    \bigvee_{i\le 2d+1}\R^2_i\vee
    \left.\left.\dots \vee
    \bigvee_{i\le 2d+1}\R^{(d-k)-1}_i\right)\right)[\mathbf{s}/\mathbf{v}]=1.
\end{multline*}
Thus, by de Morgan's law,
\begin{multline*}
    \left(\exists \mathbf{v}_{2k+2}
    \forall \mathbf{v}_{2k+3}\dots
    \forall \mathbf{v}_{2d+1}\neg
    \left(\neg\gamma\wedge \bigwedge_{i\le 2d+1}\neg \R_i\wedge\right.\right.\\  
    \bigwedge_{i\le 2d+1}\neg\R^2_i\wedge
    \left.\left.\dots \wedge
    \bigwedge_{i\le 2d+1}\neg\R^{(d-k)-1}_i\right)\right)[\mathbf{s}/\mathbf{v}]=1.
\end{multline*}
Note that if $\R^{d-k}_{2k+2}[\mathbf{s}/\mathbf{v}]=1$, then the conclusion of the lemma is true. Suppose that $\R^{d-k}_{2k+2}[\mathbf{s}/\mathbf{v}]=0$. Thus, by equation~\eqref{R definition},
\begin{equation*}
    \left(\neg\gamma\wedge \bigwedge_{i\le 2d+1}\neg \R_i\wedge  
    \bigwedge_{i\le 2d+1}\neg\R^2_i\wedge\right. 
    \left.\dots \wedge
    \bigwedge_{i\le 2d+1}\neg\R^{(d-k)-1}_i\right)[\mathbf{s}/\mathbf{v}]=0.
\end{equation*}
Hence, again by de Morgan's law,
\begin{equation*}
    \left(\gamma\vee \bigvee_{i\le 2d+1} \R_i\vee  
    \bigvee_{i\le 2d+1}\R^2_i\vee\right. 
    \left.\dots \vee
    \bigvee_{i\le 2d+1}\R^{(d-k)-1}_i\right)[\mathbf{s}/\mathbf{v}]=1.
\end{equation*}
Therefore, either $\gamma[\mathbf{s}/\mathbf{v}]=1$ or, by Definition~\ref{d-order responsibility},  there is positive $d'\le d-k$ and an agent $i\le 2d+1$ who is $d'$-order responsible under action profile $\mathbf{s}$.
\end{proof}

\end{document}